\documentclass[11pt]{article}
\usepackage{graphics,epsfig,epsf,psfrag}
\usepackage{url}
\usepackage{amsmath,amsfonts,amssymb,amsthm}
\usepackage{bbm}
\usepackage{graphicx,lscape,rotate}
\usepackage{dashrule} % To make dotted/dashed lines
\usepackage{epstopdf}
\usepackage{algorithmic, algorithm}
\usepackage{enumerate, framed}
\usepackage{colortbl}
\usepackage{pstricks,pst-node,pst-tree}  % Plotting Trees
\usepackage[round]{natbib}
\usepackage[colorlinks=true,urlcolor=blue,citecolor=red,linkcolor=red,bookmarks=true]{hyperref}

\usepackage{graphics, hyperref}  % <-- special graphics package

\let\BLS=\baselinestretch
\makeatletter
\newcommand{\singlespacing}{\let\CS=\@currsize\renewcommand{\baselinestretch}{1}\small\CS}
\newcommand{\doublespacing}{\let\CS=\@currsize\renewcommand{
\baselinestretch}{1.5}\small\CS}
\newcommand{\normalspacing}{\let\CS=\@currsize\renewcommand{\baselinestretch}{\BLS}\small\CS}
\makeatother

\newtheorem{proposition}{Proposition}[section]
\newtheorem{remark}{Remark}%[section]
\newtheorem*{remark*}{Remark} %For having unnumbered remarks.

\providecommand{\keywords}[1]{\textbf{\textit{Keywords:~}} #1}

\newcommand{\bfbeta}{\mbox{\boldmath $\beta$}}

\newcommand{\bfSigma}{\mbox{\boldmath $\Sigma$}}

\newcommand{\bflambda}{\mbox{\boldmath $\lambda$}}

\newcommand{\bfgamma}{\mbox{\boldmath $\gamma$}}

\newcommand{\bfepsilon}{\mbox{\boldmath $\varepsilon$}}

\newcommand{\X}{\mathbf{X}}
\newcommand{\x}{\mathbf{x}}
\newcommand{\y}{\mathbf{y}}

\newcommand{\diag}{{\mbox{diag}}}
\newcommand{\tr}{{\mbox{tr}}}
\newcommand{\expit}{{\mbox{expit}}}
\DeclareMathOperator*{\argmin}{\arg\!\min}

\begin{document}

% \doublespacing % <-- if you want to doublespace your document

\title{Weighted Orthogonal Components Regression Analysis}

\author{\textbf{Xiaogang Su}\footnote{Email:
\texttt{xsu@utep.edu}} \\
Department of Mathematical Sciences \\
University of Texas, El Paso, TX 79968 \vspace{.1in} \\
~\textbf{Yaa Wonkye} \\
Department of Mathematics and Statistics, \\
Bowling Green State University, Bowling Green, OH 43403 \vspace{.1in} \\
\textbf{ Pei Wang} and  \textbf{Xiangrong Yin} \\
Department of Statistics \\
University of Kentucky, Lexington, KY 40536}
\date{} \maketitle

\renewcommand{\abstractname}{\large Abstract}
\begin{abstract}
{\normalsize In the multiple linear regression setting, we propose
a general framework, termed weighted orthogonal components
regression (WOCR), which encompasses many known methods as special
cases, including ridge regression and principal components
regression. WOCR makes use of the monotonicity inherent in
orthogonal components to parameterize the weight function. The
formulation allows for efficient determination of tuning
parameters and hence is computationally advantageous. Moreover,
WOCR offers insights for deriving new better variants.
Specifically, we advocate weighting components based on their
correlations with the response, which leads to enhanced predictive
performance. Both simulated studies and real data examples are
provided to assess and illustrate the advantages of the proposed
methods.}
\end{abstract}


\section{Introduction}

Consider the typical multiple linear regression setting where the
available data $\mathcal{L} := \{(y_i, \mathbf{x}_i):
i=1, \ldots, n\}$ consist of $n$ i.i.d.~copies of the continuous response $y$ and the predictor vector $\x \in \mathbb{R}^p.$ Without loss of generality (WLOG), we assume
$y_i$'s are centered and $x_{ij}$'s are standardized throughout
the article. Thus the intercept term is presumed to be 0 in linear
models, for which the general form is given by $\mathbf{y} =
\mathbf{X} \bfbeta + \bfepsilon$ with $\mathbf{y}=\left( y_i
\right)$ and $\bfepsilon \, \sim \, \left(\mathbf{0}, \sigma^2
\mathbf{I}_n \right).$ For the sake of convenience, we sometimes
omit the subscript $i.$ When the $n \times p$ design matrix
$\mathbf{X}$ is of full column rank $p$, the ordinary least
squares (OLS) estimator $\widehat{\bfbeta} = (\mathbf{X}^T
\mathbf{X})^{-1} \mathbf{X}^T \mathbf{y}$, as well as its
corresponding predicted value $\hat{y}(\x) = \mathbf{x}^T
\widehat{\bfbeta}$ at a new observation $\mathbf{x}$, enjoys many
attractive properties.

However, OLS becomes problematic when $\mathbf{X}$ is
rank-deficient, in which case the Gram matrix
$\mathbf{X}^T\mathbf{X}$ is singular. This may happen either
because of multicollinearity when the predictors are
highly correlated or because of high dimensionality
when $p \gg n.$ A wealth of proposals have been made to combat the
problem. Besides others, we are particularly concerned with a
group of techniques that include ridge regression (RR;
\citeauthor{Hoerl.1970}, \citeyear{Hoerl.1970}), principal
components regression (PCR; \citealp{Massy.1965}), partial least
squares  regression (PLSR; \citeauthor{Wold.1966},
\citeyear{Wold.1966} \& \citeyear{Wold.1978}), and continuum
regression (CR; \citealp{Stone.1990}). One common feature of these
approaches lies in the fact that they first extract
orthogonal or uncorrelated components that are linear combinations
of $\mathbf{X}$ and then regress the response directly on the
orthogonal components. The number of orthogonal components doesn't
exceed $n$ and $p$, hence reducing the dimensionality. This is the
key how these types of methods approach high-dimensional or
multicollinear data.

In this article, we first introduce a general framework, termed
weighted orthogonal components regression (WOCR), which puts the
aforementioned methods into a unified class. Compared to the
original predictors in $\X$, there is a natural ordering in the
orthogonal components. This information allows us to parameterize
the weight function in WOCR with low-dimensional parameters, which
are essentially the tuning parameters, and estimate the tuning
parameters via optimization. The WOCR formulation also facilitates
convenient comparison of the available methods and suggests their
new natural variants by introducing more intuitive weight
functions.

We shall restrict our attention to PCR and RR models. The
remainder of the article is organized as follows. In Section
\ref{sec-WOCR}, we introduce the general framework of WOCR.
Section \ref{sec-WOCR-examples} exemplifies the applications of
WOCR with RR and PCR. More specifically, we demonstrate how WOCR
formulation can be used to estimate the tuning parameter in RR and
select the number of principal components in PCR, and then
introduce their better variants on the basis of WOCR. Section
\ref{sec-simulation} presents numerical results from simulated
studies that are designed to illustrate and assess WOCR and make
comparisons with others. We also provide real data illustrations
in Section \ref{sec-real-data}. Section \ref{sec-discussion}
concludes with a brief discussion, including the implication of
WOCR on PLSR and CR models.

\section{Weighted Orthogonal Components Regression (WOCR)}
\label{sec-WOCR}

Denote $m=\mbox{rank}(\X)$ so that $m \leq (p \wedge n).$ Let
$\{\mathbf{u}_1, \ldots, \mathbf{u}_m\}$ be the orthogonal
components extracted in some principled way, satisfying that
$\mathbf{u}_j^T \mathbf{u}_{j'} = 0$ if $j\neq j'$ and 1
otherwise. Here $\{\mathbf{u}_j \}_{j=1}^m$ forms an orthonormal
basis of the column space of $\X$, $\mathcal{C}(\X) = \{\X
\mathbf{a}:~ \mbox{for some~} \mathbf{a} \in \mathbb{R}^p\}.$
Since $\mathbf{u}_j \in \mathcal{C}(\X)$, suppose $\mathbf{u}_j =
\X \mathbf{a}_j$ for $j=1, \ldots, m.$ The condition
$\mathbf{u}_j^T \mathbf{u}_{j'} = 0$ implies that $\mathbf{a}_j^T
\X^T \X \mathbf{a}_{j'} =0$, i.e., vectors $\mathbf{a}_j$ and
$\mathbf{a}_{j'}$ are $\X^T \X$ orthogonal, which implies that
$ \mathbf{a}_j$ and $\mathbf{a}_{j'}$ are orthogonal if, furthermore, $\mathbf{a}_j$ or $\mathbf{a}_{j'}$ is an eigenvector of $\X^T \X$ associated with a non-zero eigenvalue. In matrix form, let
$\mathbf{U}_{n \times m} = [\mathbf{u}_1, \ldots, \mathbf{u}_m]$,
$\mathbf{W}_{m \times m}=\mbox{diag}(a_j),$ and $\mathbf{A}_{p
\times m} = [\mathbf{a}_1, \ldots, \mathbf{a}_m]$. We have
$\mathbf{U}= \mathbf{XA}$ with $\mathbf{U}^T\mathbf{U}
=\mathbf{I}_m$ but it is not necessarily true that $\mathbf{U}\mathbf{U}^T = \mathbf{I}_n.$ The
construction of matrix $\mathbf{A}$ may (e.g., in RR and PCR) or
may not (e.g., in PLSR and CR) depend on the response
$\mathbf{y};$ again, our discussion will be restricted to the former scenario. It is worth noting that extracting $m$ components
reduces the original $n \times p$ problem into an $n \times m$
(with $m \leq n)$ problem, hence making automatic dimension
reduction.

\subsection{Model Specification}
The general form of a WOCR model can be conveniently expressed in
terms of the fitted vector
\begin{equation}
\tilde{\mathbf{y}} = \sum_{j=1}^m w_j \gamma_j \mathbf{u}_j,
\label{fitted-WOCR}
\end{equation}
where $\gamma_j = \langle \y , \mathbf{u}_j \rangle$ is the
regression coefficient and $0 \leq w_j \leq 1$ is the weight for
the $j$-th orthogonal component $\mathbf{u}_j$.  In matrix form,
(\ref{fitted-WOCR}) becomes
\begin{equation}
\tilde{\mathbf{y}} = \mathbf{UWU}^T \mathbf{y} = \mathbf{XAWU}^T
\mathbf{y}. \label{fitted-WOCR-matrix}
\end{equation}
We will see that RR, PCR, and many others are all special cases of
the above WOCR specification, with different choices of
$\{\mathbf{u}_j, w_j\}.$ For example, if $w_j =1$ or
$\mathbf{W}=\mathbf{I}_m$, then (\ref{fitted-WOCR}) amounts to the
least square fitting since $\mathcal{C}(\mathbf{U}) = \mathcal{C}(\mathbf{X}).$

This WOCR formulation allows us to conveniently study its general
properties. It follows immediately from (\ref{fitted-WOCR-matrix})
that the associated hat matrix $\mathbf{H}$ is
\begin{equation}
\mathbf{H} = \mathbf{UWU}^T = \mathbf{XAWU}^T. \label{hat-WOCR}
\end{equation}
The resultant sum of square errors (SSE) is given by $\mbox{SSE}
~=~  \parallel \tilde{\y} - \y \parallel^2  ~=~ \y^T (\mathbf{I}_n
- \mathbf{H})^2 \y.$ Note that $\mathbf{H}$ is not an idempotent
or projection matrix in general, neither is $(\mathbf{I} -
\mathbf{H}).$ Instead,
$$ (\mathbf{I-H})^2 = \mathbf{I-2 H + H^2} = \mathbf{I} -
\mathbf{U}(2 \mathbf{W} - \mathbf{W}^2) \mathbf{U}^T.$$ The
diagonal matrix $(2 \mathbf{W} - \mathbf{W}^2)$ has diagonal
element $\{1-(1-w_j)^2\}$. Therefore,
\begin{eqnarray}
\mbox{SSE} &=& \mathbf{y}^T \y - \y^T \mathbf{U} \diag\left\{ 1-
(1-w_j)^2 \right\} \mathbf{U}^T \y \nonumber \\
&=& \parallel \y  \parallel^2 - \sum_{j=1}^m \left( 2 w_j - w_j^2
\right) \gamma_j^2. \label{SSE}
\end{eqnarray}
From (\ref{fitted-WOCR-matrix}), the WOCR estimate of $\bfbeta$ is
\begin{equation} \widetilde{\bfbeta} = \mathbf{AWU}^T \y.
\label{beta-WOCR}
\end{equation}
It follows that, given new data matrix $\mathbf{X}'$, the
predicted vector is
\begin{equation}
\widetilde{\textbf{y}}'=\mathbf{X}'\widetilde{\bfbeta} =
\mathbf{X'AWU}^T \y.
\label{y-pred}
\end{equation}
Although not further pursued here, many other quantities and
properties of WOCR can be derived accordingly with the generic
form, including $\mbox{E} \parallel \widetilde{\bfbeta} - \bfbeta
\parallel^2$ as studied in \cite{Hoerl.1970} and
\citet{Hwang.2003}.

\subsection{Parameterizing the Weights}

The next important component in specifying WOCR is to parameterize the
weights in $\mathbf{W}$ in a principled way. The key motivation
stems from the observation that, compared to the original
regressors in $\X$, the orthogonal components in $\mathbf{U}$ are
naturally ordered in terms of some measure. This ordering may be
attributed to some specific variation that each $\mathbf{u}_j$ is
intended to account for. Another natural ordering is based on the
coefficients $\{|\gamma_j|\}_{j=1}^m.$ Because of orthogonality,
the regression coefficient $\gamma_j$ remains the same for
$\mathbf{u}_j$ in both the simple regression and multiple
regression settings.

This motivates us to parameterize the weights $w_j$ based on the
ordering measure. It is intuitive to assign more weights to more
important components. To do so, $w_j$ can be specified as a
function monotone in the ordering measure and parameterized with a
low-dimensional vector $\bflambda.$ Two such examples are given in
Figure \ref{fig01}. Among many other choices, the usage of sigmoid
functions will be advocated in this article because they provide a
smooth approximation to the 0-1 threshold indicator function that
is useful for the component selection purpose and they are also
flexible enough to adjust for achieving improved prediction
accuracy. In general, we denote $w_j = w_j(\bflambda).$ The vector
$\bflambda$ in the weight function are essentially the tuning
parameters. This parameterization conveniently expands these
conventional modeling methods by providing several natural WOCR
variants that are more attractive, as illustrated in the next
section.

\begin{figure}[h]
\centering
  \includegraphics[scale=0.6, angle=270]{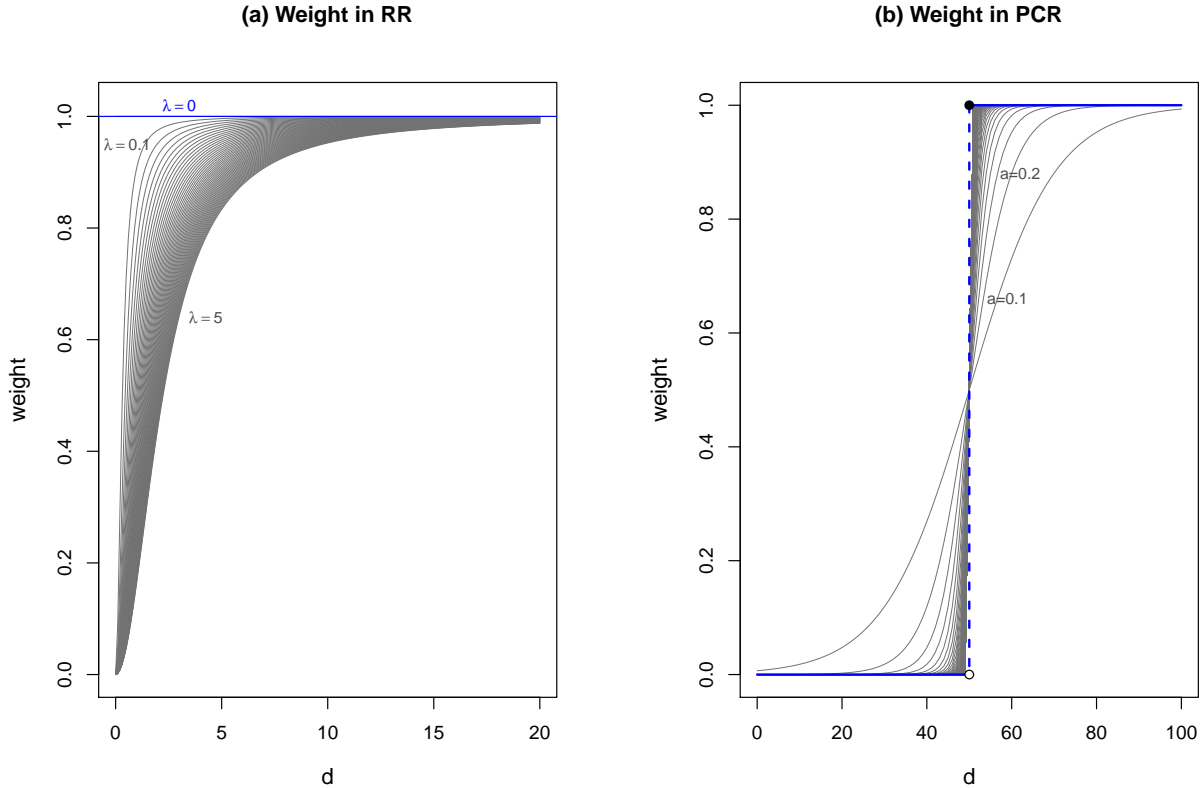}
  \caption{Plot of the weights used in ridge regression (RR)
  and principal components regression (PCR) as a function of the
  singular values $d_j$ of $\X$: (a) $w(d) = d^2/(d^2 + \lambda)$
  in RR for $\lambda = 0.0, 0.1, 0.2, \ldots, 5.0;$ (b) the discrete
  threshold $w(d) = I(x \geq c)$ in PCR with $c=50.0,$ approximated
  with the expit weight $w(d) = \expit\{ a \, (d-c) \}$ for
  $a= \{0.1, 0.2, \ldots, 50.0 \}$. \label{fig01}}
\end{figure}

Determining the tuning parameters $\bflambda$ is yet another
daunting task. In common practice, one fits the model for a number
of fixed tuning parameters and then resorts to cross-validation or
a model selection criterion to compare the model fittings. This
can be computationally intensive, especially with big data. When a
model selection criterion is used, WOCR provides a computationally
efficient way of determining the tuning parameter $\bflambda$.
% and a natural $M$-estimation mechanism for making inference on $\bflambda$.
The key idea is to plug the specification (\ref{fitted-WOCR}) in a
model selection criterion and optimize with respect to
$\bflambda.$ Depending on the scenarios, commonly used model
selection criteria include the Akaike information criterion (AIC;
\citealp{Akaike.1974}), the generalized cross-validation (GCV;
\citealp{Golub.1979}), and the Bayesian information criterion
(BIC; \citealp{Schwarz.1978}). What is involved in these model
selection criteria are essentially SSE and the degrees of freedom
(DF). A general form of SSE is given by (\ref{SSE}). For DF, we
follow the generalized definition by \citet{Efron.2004}:
\begin{equation}
\mbox{DF}(\bflambda) = \mathbb{E} \left\{ \tr(d \hat{\y} / d
\mathbf{y}) \right\}. \label{DF-Efron}
\end{equation}
If neither the components $\mathbf{U}$ nor the weights $w_j$
depends on $\mathbf{y},$ then DF, often termed as the effective
degrees of freedom (EDF) in this scenario, is computed as
\begin{equation} \label{EDF}
\mbox{EDF} = \tr(\mathbf{H}) = \tr(\mathbf{UWU}^T) =
\tr(\mathbf{WU}^T\mathbf{U}) = \tr(\mathbf{W}) = \sum_{j=1}^m w_j.
\end{equation}
With either components $\mathbf{U}$ or the weights $w_j$ depends
on $\mathbf{y},$ the computation of DF is more difficult and will
be treated on a case-by-case basis.

The specific forms of GCV, AIC, and BIC can be obtained
accordingly. We treat the model selection as an objective function
for $\bflambda$. The best tuning parameter $\hat{\bflambda}$ can
then be estimated by optimization. Since $\bflambda$ is of low
dimension, the optimization can be solved efficiently. This saves
the computational cost in selecting the tuning parameter.

\section{WOCR Examples}
\label{sec-WOCR-examples}
We show how several conventional models relate to WOCR with
different weight specifications and different ways of constructing
the orthogonal components $\mathbf{U}=\mathbf{XA}$ and then how
the WOCR formulation can help improve and expand them.
\iffalse
Whether or
not the response $\mathbf{y}$ has been used in defining the matrix
$\mathbf{A}$ is crucial, because, if this is the case (e..g, in
PLSR and CR), the associated statistical inference becomes much
more convoluted.
\fi
In this section, we first discuss how WOCR helps determine the
optimal tuning parameter $\lambda$ in ridge regression and make
inference accordingly. Next, we show that WOCR facilitates an
efficient computational method for selecting the number of
components in PCR. The key idea is to approximate the 0-1
threshold function with a smooth sigmoid weight function. Several
natural variants of RR and PCR that are advantageous in predictive
modeling are then derived within the WOCR framework.

\subsection{Pre-Tuned Ridge Regression}
\label{sec-ridge}

The ridge regression \citep{Hoerl.1970} can be formulated as a
penalized least square optimization problem
$$ \min_{\bfbeta} ~ \parallel \mathbf{y} - \mathbf{X} \bfbeta
\parallel^2 + \lambda \, \parallel \bfbeta \parallel^2,$$
with the tuning parameter $\lambda.$ The solution yields the ridge
estimator $\widehat{\bfbeta}_R = \left( \mathbf{X}^T \mathbf{X} +
\lambda \, \mathbf{I}_p \right)^{-1} \mathbf{X} \mathbf{y}.$

The singular value decomposition (SVD) of data matrix $\mathbf{X}$
offers a useful insight into RR \citep[see, e.g.,][]{Hastie.2009}.
Suppose that the SVD of $\mathbf{X}$ is given by
\begin{equation}
\mathbf{X} = \mathbf{UDV}^T = \sum_{j=1}^m d_j \mathbf{u}_j
\mathbf{v}_j, \label{SVD-X}
\end{equation}
where both $\mathbf{U}=\left[\mathbf{u}_1, \ldots, \mathbf{u}_m
\right] \in \mathbb{R}^{n \times m}$ and
$\mathbf{V}=\left[\mathbf{v}_1, \ldots, \mathbf{v}_m \right] \in
\mathbb{R}^{p \times m}$ have orthonormal column vectors that form
an orthonormal basis for the column space
$\mathcal{C}(\mathbf{X})$ and the row space
$\mathcal{C}(\mathbf{X}^T)$ of $\mathbf{X}$, respectively, and
matrix $\mathbf{D} = \mbox{daig} \left( d_j \right)$ with singular
values satisfying $d_1 \geq d_2 \geq \cdots \geq d_m >0.$ Noticing
that $\mathbf{X}^T\mathbf{X} = \mathbf{V}
\mathbf{D}^2 \mathbf{V}^T,$ the column vectors of $\mathbf{V}$
yield the principal directions. Since $\mathbf{X} \mathbf{v}_j =
d_j \mathbf{u}_j$,  it can be seen that $\mathbf{u}_j$ is the
$j$-th normalized principal component.

The fitted vector in RR conforms well to the general form
(\ref{fitted-WOCR}) of WOCR, as established by the following
proposition.
\begin{proposition}
Regardless of the magnitude of $\{n, p, m\}$, the fitted vector
$\hat{\mathbf{y}}_R = \mathbf{X} \widehat{\bfbeta}_R$ in ridge
regression can be written as
\begin{equation}
\hat{\mathbf{y}}_R = \mathbf{X} \left( \mathbf{X}^T \mathbf{X} +
\lambda \, \mathbf{I}_p \right)^{-1} \mathbf{X} \mathbf{y}  =
\mathbf{UWU}^T \mathbf{y} = \sum_{j=1}^m \frac{d_j^2}{d_j^2 +
\lambda} \, \langle \mathbf{y}, \mathbf{u}_j \rangle \,
\mathbf{u}_j, \label{RR-WOCR}
\end{equation}
with $\mathbf{W} = \diag \left\{ w_j\right\}$ and $w_j = d_j^2
/(d_j^2 + \lambda)$ for $j=1, \ldots, m.$ \label{prop-fitted-RR}
\end{proposition}
\begin{proof}% [Proof of Proposition (\ref{prop-fitted-RR})]~~ 
The proof when $m=p$ (i.e., $p < n$ and hence $\mathbf{V}^{-1}=\mathbf{V}$) can be found in, e.g., \citet{Hastie.2009}. We consider the general case including the $p \gg n$ scenario. With the
general SVD form (\ref{SVD-X}) of $\mathbf{X}$, we have
$\mathbf{U}^T \mathbf{U} = \mathbf{V}^T \mathbf{V}= \mathbf{I}_m$,
but it is not necessarily true that $\mathbf{U} \mathbf{U}^T =
\mathbf{I}_n$, nor for $\mathbf{V} \mathbf{V}^T = \mathbf{I}_p.$

First, plugging the SVD of $\mathbf{X}$ into
$\hat{\mathbf{y}}_{\tiny{R}}$ yields
\begin{equation}
\hat{\mathbf{y}}_{\tiny{R}} = \mathbf{UDV}^T \left(
\mathbf{VD}^2\mathbf{V}^T + \lambda \mathbf{I}_p \right)^{-1}
\mathbf{VDU}^T \mathbf{y}.
 \label{yhat-RR}
\end{equation}
Define $\mathbf{V}' = \left[ \mathbf{v}_1, \cdots, \mathbf{v}_m,
\mathbf{v}_{m+1}, \ldots, \mathbf{v}_p \right] \in \mathbf{R}^{p
\times p}$ by completing an orthonormal basis for $\mathbb{R}^p.$
Hence $\mathbf{V}$ is invertible with $\mathbf{V}^{-1} =
\mathbf{V}^T.$ Also define $\mathbf{U}_0 = \left[ \mathbf{U},
\mathbf{O} \right] \in \mathbf{R}^{p \times p}$ and $\mathbf{D}_0
= \diag\{ d_1, \ldots, d_m, 0, \ldots, 0 \} \in \mathbf{R}^{p
\times p}$ by appending 0 matrix $\mathbf{O}$ or components to
$\mathbf{U}$ and $\mathbf{D}.$ Then it can be easily checked that
$\hat{\mathbf{y}}_{\tiny{R}}$ in (\ref{yhat-RR}) can be rewritten
as
\begin{eqnarray*}
\hat{\mathbf{y}}_{\tiny{R}} &=& \mathbf{U_0D_0V'}^T \left(
\mathbf{V'D_0}^2\mathbf{V'}^T + \lambda \mathbf{I}_p \right)^{-1}
\mathbf{V'D_0U_0}^T \mathbf{y} \\
&=& \mathbf{U_0D_0V'}^T \left\{ \mathbf{V'} \left(\mathbf{D}_0^2 +
\lambda \mathbf{I}_p \right) \mathbf{V'}^T \right\}^{-1}
\mathbf{V'D_0U_0}^T \mathbf{y} \\
&=& \mathbf{U_0D_0V'}^T  \mathbf{V'} \left(\mathbf{D}_0^2 +
\lambda \mathbf{I}_p \right)^{-1} \mathbf{V'}^{-1} \mathbf{V'}
\mathbf{D_0U_0}^T \mathbf{y} \\
&=& \mathbf{U_0D_0} \left(\mathbf{D}_0^2 + \lambda \mathbf{I}_p
\right)^{-1} \mathbf{D_0U_0}^T \mathbf{y} \\
&=& \mathbf{UWU}^T \mathbf{y}
\end{eqnarray*}
with $\mathbf{W} = \diag\left\{d_j^2/(d_j^2 + \lambda) \right\}.$
\end{proof}

One natural ordering of the principal components $\mathbf{u}_j$s
is based on their associated singular values $d_j$. Hence, the
weight function $w_j = w(d_j; \lambda)=d_j^2/(d_j^2 + \lambda)$ is
monotone in $d_j$ and parameterized with one single parameter
$\lambda.$ See Figure \ref{fig01}(a) for a graphical illustration
of this weight function. In view of $\mathbf{XV} = \mathbf{UD}$,
matrix $\mathbf{A}$ in WOCR is given as $\mathbf{A} =
\mathbf{VD}^{-1}.$

Since RR is most useful for predictive modeling without
considering component selection, GCV is an advisable criterion for
selecting the best tuning parameter $\hat{\lambda}.$ With our WOCR
approach, we first plugging (\ref{RR-WOCR}) into GCV to form an
objective function for $\lambda$ and then optimize it with respect
to $\lambda$. On the basis of (\ref{SSE}) and (\ref{EDF}), the
specific form of $\mbox{GCV}(\lambda)$ is given up to some
irrelevant constant, by
\begin{equation}
\mbox{GCV}(\lambda) ~\propto~ \frac{SSE}{(n- EDF)^2} =
\frac{\parallel \mathbf{y}
\parallel^2 - \sum_{j=1}^m  (w_j^2 - 2 w_j) \, \gamma_j^2}{(n - \sum_{j=1}^m
w_j)^2}, \label{GCV-I}
\end{equation}
GCV has a wide applicability even in the ultra-high
dimensions. Alternatively, AIC can be used instead. If $ \lim_{n \rightarrow \infty} m/n = 0,$ GCV is asymptotically equivalent to $\mbox{AIC}(\lambda) ~\propto~ n \,
\ln (\mbox{SSE}) + 2 \cdot \mbox{EDF}.$

The best tuning parameter in RR can be estimated as $\hat{\lambda}
= \argmin_{\lambda} \mbox{GCV}(\lambda).$ Bringing $\hat{\lambda}$ back to $\widehat{\bfbeta}_R$ yields the final RR estimator. Since the tuning parameter is determined beforehand, we call this method `pre-tuning'. We denote
this pre-tuned RR method as $\mbox{RR}(d; \lambda),$ where the
first argument $d$ indicates the ordering on which basis the
components are sorted and the second argument indicates the tuning
parameter $\lambda$. We shall use this as a generic notation for
other new WOCR models. As we shall demonstrate with simulation in
Section \ref{sec-simulation-ridge}, $\mbox{RR}(d; \lambda)$
provides nearly identical fitting results to RR; however,
pre-tuning dramatically improves the computational efficiency,
especially when dealing with massive data.

\begin{remark}
One statistically awkward issue with regularization is selection
of the tuning parameter. First of all, this is a one-dimensional
optimization problem, yet done in a poor way in current practice by
selecting a grid of values and evaluating the objective function
at each value. The pre-tuned version helps amend this deficiency.
Secondly, although the tuning parameter $\lambda$ is often
selected in a data-adaptive way and hence clearly is a statistic,
no statistical inference is made for the tuning parameter unless
within the Bayesian setting. The above pre-tuning method yields a
convenient way of making inference on $\lambda.$ Since the
objective function $\mbox{GCV}(\lambda)$ is smooth in $\lambda$,
the statistical properties of $\hat{\lambda}$ follow well through
standard M-estimation arguments. However, this is not the theme of
WOCR, thus we shall not pursue further.
\end{remark}

\subsection{Pre-Tuned PCR}
\label{sec-PCR}

PCR regresses the response on the first $k$ ($1 \leq k \leq m$)
principal components as given by the SVD of $\mathbf{X}$ in
(\ref{SVD-X}). The fitted vector in PCR can be rewritten as
$$ \hat{\y}_{PCR} = \sum_{j=1}^k  \langle \y, \mathbf{u}_j \rangle
\mathbf{u}_j = \sum_{j=1}^m \delta_j \, \gamma_j \,
\mathbf{u}_j,$$ where $\gamma_j =\langle \y, \mathbf{u}_j \rangle$
and $\delta_j = I(j \leq k)$ for $j=1, \ldots, m.$ Clearly, PCR
can be put in the WOCR form with $w_j = \delta_j.$ Conventionally,
the ordering of principal components is aligned with the singular
values $\{d_j\};$ thus we may rewrite $\delta_j =\delta(d_j; c)=
I(d_j \geq c)$ with a threshold value $c = d_k$ if $k$ is known.
Either the number of components $k$ or the threshold $c$ is the
tuning parameter. Selecting the optimal $k$ by examining many PCR
models is a discrete process.

To facilitate pre-tuning, we replace the indicator weight
$\delta(x; c) = I(x \geq c)$ with a smooth sigmoid function. While
many other choices are available, it is convenient to use the
logistic or expit function $ \pi(x)= \mbox{expit}(x) =  \{ 1 +
\exp(-x)\}^{-1} $ so that
\begin{equation}
w_j = \pi(d_j; a, c) = \mbox{expit}\{a (d_j -c)\}. \label{expit}
\end{equation}
Figure \ref{fig01}(b) plots $\mbox{expit}\{ a (x - c)\}$ with
$c=50.0$ for different choices of $a$. It can be seen that a
larger $a$ value yields a better approximation to the indicator
function $I(x \geq 0),$ while a smaller $a$ yields a smoother
function which is favorable for optimization. In order to emulate
PCR, the parameter $a$ can be fixed \textit{a priori} at a
relatively large value. Our numerical studies shows that the performance of the method is quite robust with respect to the choice of $a$. On that basis, we recommend fixing $a$ in the range of $[10, 50].$

Since PCR involves selection of the optimal number of PCs, BIC,
given by $\mbox{BIC}(\bflambda) ~\propto~ n \, \ln (\mbox{SSE}) +
\ln(n) \cdot \mbox{DF},$ is selection-consistent \citep{Yang.2005}
and often has a superior empirical performance in variable
selection. The hat matrix $\mathbf{H}$ in PCR is idempotent, so is
$\mathbf{I}_n- \mathbf{H}.$ Thus the SSE can be reduced a little
bit as $\mathbf{y}^T(\mathbf{I}_n- \mathbf{H}) \y$, which then can
be approximated by substituting $\delta(d_j; c)$ with $\pi(d_j; a,
c).$ The DF can be approximately in a similar way as $DF = k =
\sum_j \delta(d_j; c) \approx \sum_j \pi(d_j; a, c).$ This results
in the following form for BIC
\begin{equation}
\mbox{BIC}(c) ~\propto~ n \, \ln \left( \parallel \y
\parallel^2 - \sum_{j=1}^m w_j \gamma_j^2 \right) + \ln(n) \,
\sum_{j=1}^m w_j, \label{BIC-c-PCR}
\end{equation}
which is treated as an objective function of $c$. We estimate the
best cutoff point $\hat{c}$ by optimizing $\mbox{BIC}(c)$ with
respect to $c$. This is a one-dimensional smooth optimization problem
with a search range $c \in [d_1, d_m].$ Once $\hat{c}$ is available, we use it as a threshold to select the components and fit a regular PCR. We denote this pre-tuned PCR approach as $\mbox{PCR}(d; a).$ Compared to the discrete
selection in PCR, $\mbox{PCR}(d; a)$ is computationally more
efficient. Furthermore, it performs better in selecting the
correct number of components, especially when weak signals are
present. This is an additional benefit of smoothing as opposed to the discrete selection process in PCR, as we will demonstrate with simulation.

\subsection{WOCR Variants of RR and PCR Models}
\label{sec-variants}

Not only can many existing models be cast into the WOCR framework,
but it also suggests new favorable variants. We explore some of them.
One first variant of PCR is leave both $a$ and $c$ free in
(\ref{BIC-c-PCR}). More specifically, we first obtain $(\hat{a},
\hat{c}) = \argmin_{a, c} \mbox{BIC}(a, c)$ and then compute the
WOCR fitted vector in (\ref{fitted-WOCR}) with weight $w_j =
\mbox{exp}\{ \hat{a}(d_j - \hat{c}) \}$ for $j=1, \ldots, m.$ This
will give PCR more flexibility and adaptivity and hence may lead
to improved predictive power. In this approach, selecting components is no longer a concern; thus GCV or AIC can be used as the objective function instead. We denote this approach as $\mbox{PCR}(d_j; a, c).$

The principal components are constructed independently from the
response. \citet{Artemiou.2009} and \citet{Ni.2011} argued that
the response tends to be more correlated with the leading
principal components; this is usually not the case in reality,
however. See, e.g., \citet{Jollife.1982} and \citet{Hadi.1998} for
real-life data illustrations. Nevertheless, there has not been a
principled way to deal with this issue in PCR. WOCR can provide a
convenient solution: one simply bases the ordering of
$\mathbf{u}_j$ on the regression coefficients $\gamma_j$ and
defines the weights $w_j$ via a monotone function of $|\gamma_j|$
or, preferably, $\gamma^2_j$. However, doing so will induce
dependence on the response to the weights. As a result, the
associated DF has to be computed differently, as established in
Proposition \ref{prop-df}.
\begin{proposition}
Suppose that the WOCR model (\ref{fitted-WOCR}) has orthogonal
components $\mathbf{u}_j$ constructed independently of $\y$ and
weights $w_j = w(\gamma^2_j; \bflambda),$ where $w(\cdot)$ is a
smooth monotonically increasing function and $\bflambda$ is the
parameter vector. Its degrees of freedom (DF) can be estimated as
\begin{equation}
\widehat{\mbox{DF}} = \sum_{j=1}^m \, (2 \gamma_j^2 \dot{w}_j +
w_j), \label{DF-RR-PCR-variants}
\end{equation}
where $\dot{w}_j = d w(\gamma^2_j; \bflambda) / d(\gamma_j^2).$
\label{prop-df}
\end{proposition}
\begin{proof}%[Proof of Proposition (\ref{prop-df})]~~ 
The WOCR model in this case is $\hat{y} = \sum_{j=1}^m w_j \gamma_j
\mathbf{u}_j,$ with $\gamma_j = \mathbf{u}_j^T \y$ and $w_j =
w(\gamma_j^2; \bflambda).$ It follows by chain rule that $$\frac{d
\hat{\y}}{d \y} \, = \, \sum_{j=1}^n (2 \gamma_j^2 \dot{w}_j +
w_j) \mathbf{u}_j \mathbf{u}_j^T \, = \, \mathbf{U} \diag(2
\gamma_j^2 \dot{w}_j + w_j) \mathbf{U}^T.
$$
Following the definition of DF by \citet{Efron.2004}, an estimate  is given by
$$ \tr \left(\frac{d \hat{\y}}{d \y} \right) \, = \, \diag(2 \gamma_j^2
\dot{w}_j + w_j) \mathbf{U}^T \mathbf{U}  \, = \, \sum_{j=1}^m (2
\gamma_j^2 \dot{w}_j + w_j),$$ which completes the proof.
\end{proof}

Clearly both PCR and RR can be benefited from this reformulation.
As a variant of RR, the weight now becomes $w_j = w(\gamma_j^2;
\lambda) = \gamma_j^2/ (\gamma_j^2 + \lambda)$ and hence
$\dot{w}_j = \lambda /(\gamma_j^2 + \lambda)^2.$ It follows that
the estimated DF is $$\widehat{\mbox{DF}} = \sum_{j=1}^m
(\gamma_j^4 + 3 \lambda \gamma_j^2)/(\gamma_j^2 + \lambda)^2.$$ The
best tuning parameter $\hat{\lambda}$ can be obtained by
minimizing GCV. Using similar notations as earlier, we denote this
RR variant as $\mbox{RR}(\gamma; \lambda).$ It is worth noting
that $\mbox{RR}(\gamma; \lambda)$ is, in fact, not a ridge
regression model. Its solution can no longer be
nicely motivated by a regularized or constrained least square
optimization problem as in the original RR. But what really
matters in these methods is the predictive power.
By directly formulating the fitted values $\hat{\y}$,
the WOCR model (\ref{fitted-WOCR}) facilitates a direct and
flexible model specification that focuses on prediction.

\renewcommand{\tabcolsep}{4.pt}
\renewcommand{\arraystretch}{1.1}
% \renewcommand{\baselinestretch}{.9}
% \begin{landscape}
\begin{table}[h]
\centering \caption{WOCR Variants of ridge regression (RR) and
principal components regression (PCR) models, both based on the
normalized principal components $\{\mathbf{u}_j: j=1, \ldots p\}$.
\label{tbl-WOCR-Models}} \vspace{.2in}
% \rule{5.5in}{.03cm}
% \begin{small}
\begin{tabular}{lclcc} \hline \hline
& Component & & Tuning & Suggested WOCR \\
Model & Ordering & Weights & Parameter & Objective Function \\ \hline
$\mbox{RR}(d; \lambda)$ & $d_j$ & $w_j = d_j^2/(d_j^2 + \lambda)$ & $\lambda$ &  $\mbox{GCV}(\lambda)$ \\
$\mbox{RR}(\gamma; \lambda)$ & $\gamma_j^2$ & $w_j = \gamma_j^2/(\gamma_j^2 + \lambda)$ & $\lambda$ &  $\mbox{GCV}(\lambda)$ \\
$\mbox{PCR}(d; c)$ & $d_j$ & $w_j = \mbox{expit}\{ a (d_j -c) \}$ with fixed $a$ & $c$ & $\mbox{BIC}(c)$ \\
$\mbox{PCR}(d; a, c)$ & $d_j$ & $w_j = \mbox{expit}\{ a (d_j -c) \}$ & $a, c$ & $\mbox{GCV}(a, c)$ \\
$\mbox{PCR}(\gamma; c)$ & $\gamma_j^2$ & $w_j = \mbox{expit}\{ a ( \gamma^2_j -c) \}$ with fixed $a$ & $c$ & $\mbox{BIC}(c)$ \\
$\mbox{PCR}(\gamma; a, c)$ & $\gamma_j^2$ & $w_j = \mbox{expit}\{ a ( \gamma^2_j -c) \}$ & $a, c$ & $\mbox{GCV}(a, c)$ \\ \hline
\end{tabular}
% \rule{8.2in}{.03cm}
% \end{small}
\end{table}
% \end{landscape}

For PCR, the weight becomes $w_j = \pi(\gamma_j^2; a, c).$  Hence,
$\dot{w}_j = a w_j (1-w_j)$ and $$\widehat{\mbox{DF}} =
\sum_{j=1}^m w_j (2ar_j^2 +1 - 2 a w_j r_j^2).$$ Depending on
whether or not we want to select components, we may fix $a$ at a
larger value or leave it free. This results in two PCR variants,
which we denote as $\mbox{PCR}(\gamma_j^2; c)$ and
$\mbox{PCR}(\gamma_j^2; a, c),$ respectively.

Table \ref{tbl-WOCR-Models} summarizes the WOCR models that we
have discussed so far. Among them, $\mbox{RR}(d_j; \lambda)$ and
$\mbox{PCR}(d_j^2; c)$ resemble the conventional RR and PCR, yet
with pre-tuning. Depending on the analytic purpose, we also
suggest a preferable objective function for each WOCR model. In general, we have recommended using GCV for predictive purposes, in which scenarios AIC can be used as an alternative. AIC is equivalent to GCV if $\lim_{n \rightarrow \infty} p/n = 0$, both being selection-efficient in the sense prescribed by \citet{Shibata.1981}. On the other hand, if selecting components is desired, using BIC is recommended.

% \iffalse
\begin{remark}
It is worth noting that the WOCR model $\mbox{PCR}(\gamma_j^2; c)$
has a close connection with the MIC (Minimum approximated
Information Criterion) sparse estimation method of
\citet{Su.2015}, \citet{Su.2016}, and \citet{Su.2017}. MIC yields
sparse estimation in the ordinary regression setting by solving a
$p$-dimensional smooth optimization problem
$$ \min_{\bfgamma}~~~ n \ln \parallel \y - \mathbf{X} \mathbf{W} \bfgamma \parallel^2 ~+~ \ln(n) \, \tr(\mathbf{W}),$$
where $\mathbf{W} = \diag \left( w_j \right)$ with diagonal element $w_j= \tanh(a \gamma_j^2)$ approximating the indicator function $I(\gamma_j \neq 0).$ Comparatively, $\mbox{PCR}(\gamma_j^2; c)$ solves a one-dimensional optimization problem
$$ \min_{c}~~~ n \ln \parallel \y - \mathbf{U} \mathbf{W} \bfgamma \parallel^2 ~+~ \ln(n) \, \tr(\mathbf{W}),$$
where $\mathbf{W} = \diag \left( w_j \right)$ with diagonal element $w_j=\expit\{a (\gamma_j^2 -c)\}$ approximating $I(\gamma_j^2 \geq c).$
The substantial simplification in $\mbox{PCR}(\gamma_j^2; c)$ is because of the orthogonality of the design matrix $\mathbf{U}.$ Hence the coefficient estimates $\bfgamma$ in multiple regression are the same as those in simple regression and can be computed ahead. Furthermore, the orthogonal regressors $\mathbf{u}_j$, i.e., the columns of $\mathbf{U}$, are naturally ordered by $\gamma_j^2$. This allows us to formulate a one-parameter smooth approximation to the indicator function $I(\gamma_j^2 \geq c),$ which induces selection of $\mathbf{u}_j$ in this PCR variant.
\end{remark}
% \fi

\subsection{Implementation: R Package \textbf{WOCR}}
\label{sec-implementation}
The proposed WOCR method is implemented in an R package
\textbf{WOCR}. The current version is hosted on GitHub at
\url{https://github.com/xgsu/WOCR}. 

The main function \texttt{WOCR()} has an argument \texttt{model=} with values in \texttt{RR.d.lambda},
\texttt{RR.gamma.lambda}, \texttt{PCR.d.c}, \texttt{PCR.gamma.c},
\texttt{PCR.d.a.c}, and \texttt{PCR.gamma.a.c}, which corresponds to the six WOCR variants as listed in Table \ref{tbl-WOCR-Models}. Among them, $\mbox{RR}(d; \lambda)$, $\mbox{RR}(\gamma; \lambda)$, $\mbox{PCR}(d; c)$, and  $\mbox{PCR}(\gamma; c)$ involves one-dimensional smooth optimization. This can be solved via the \citet{Brent.1973} method, which is conveniently available in the R function \texttt{optim()}. Owing to the nonconvex nature, dividing the search range of the decision variable can be helpful. The other two methods, $\mbox{PCR}(d; a, c)$ and $\mbox{PCR}(\gamma; a, c)$, involve two-dimensional smooth nonconvex optimization. \cite{Mullen.2014} provides a comprehensive comparison of many global optimization algorithms currently available in R \citep{R.2018}. We have followed her suggestion to choose the generalized simulated annealing method \citep{Tsallis.1996}, which is available from the R package \textbf{GenSA} \citep{Xiang.2013}.  More details of the implementation can be found from the help file of the \textbf{WOCR} package.

\section{Simulation Studies}
\label{sec-simulation}
This section presents some of the simulation studies that we have conducted to investigate the performance of WOCR models and compare them to other methods.

\subsection{Comparing Ridge Regression with $\mbox{RR}(d; \lambda)$}
\label{sec-simulation-ridge}
We first compare the conventional ridge regression with its pretuned version, i.e., $\mbox{RR}(d_j; \lambda)$. The data are generated as follows. We first simulate the design matrix $\mathbf{X} \in \mathbb{R}^{n \times p}$ from a multivariate normal distribution $N(\mathbf{0}, \, \bfSigma)$ with $\bfSigma=(\sigma_{j j'})$ and $\sigma_{j j'}=\rho^{|j-j'|}$ for $j, j'=1, \ldots, p.$ Apply SVD to extract matrix $\mathbf{U}$ and $\mathbf{D}$. Then we form the mean response as
\begin{equation}
\mbox{Model A:~~~~~~~} \mathbf{y} = \sum_{j=1}^m  b_j \mathbf{u}_j + \bfepsilon \mbox{~with~} m = p \wedge n \mbox{~and~} \bfepsilon \sim \mathcal{N}\left(\mathbf{0}, \sigma^2 \mathbf{I}_n \right),
\label{Simulation-Model}
\end{equation}
where
$$ \mathbf{b} = \left(b_j \right) = \left[m, \, m-1,\,  \ldots,  \, 1\right]^T/10. $$
For each simulated data set, we apply RR (as implemented by the R
function \texttt{lm.ridge}) and  $\mbox{RR}(d; \lambda)$, both
selecting $\lambda$ with minimum GCV.

\begin{figure}[h]
\centering
  \includegraphics[scale=0.6, angle=270]{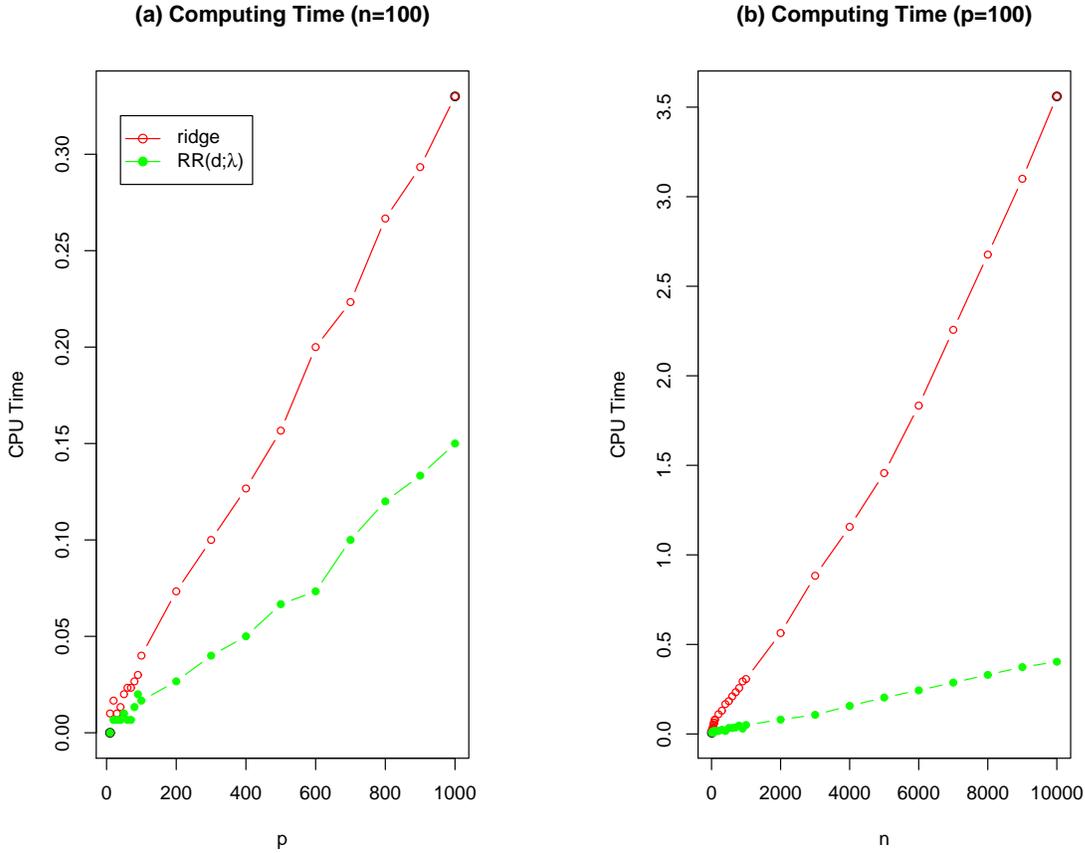}
  \caption{CPU Computing time comparison between ridge regression (RR) and its WOCR variant
  $\mbox{RR}(d; \lambda)$ in selecting the best $\lambda$ via minimum GCV:
  (a) with varying $p$ and fixed $n=100$ and (b) with varying $p$ and fixed $n=100$.
 \label{fig02}}
\end{figure}

To compare, we consider the mean square error (MSE) for
prediction. To this end, a test data set of $n'=500$ is generated
in advance. The fitted RR and $\mbox{RR}(d; \lambda)$ from each
simulation run will be applied to the test set and the MSE is
obtained accordingly. The `best' tuning parameter $\hat{\lambda}$
is also recorded. We only report the results for the setting
$\rho=0.5$, $\sigma^2 =1,$ $p=100$, $\mathbf{b} = \left(b_j
\right) = \left(p, \, p-1,\,  \ldots,  \, 1\right)^T/10.$ Two
sample sizes $n \in \{50, 500\}$ are considered. For each model
configuration, a total of $200$ simulation runs are considered.

In the simulation, we found how to specify the search points could
be a problem in the current practice of ridge regression.
Initially, we found the ridge regression gave inferior performance
compared to $\mbox{RR}(d; \lambda)$ in many scenarios. However,
after adjusting its search range, the results became nearly
identical to what $\mbox{RR}(d; \lambda)$ had. This point will be
further illustrated in Section \ref{sec-simulation-predictive}. It
is also worth noting that the minimum GCV tends to select a very
small $\lambda$ in the ultra-high dimensional case with $p > n.$

To demonstrate the computational advantages of $\mbox{RR}(d;
\lambda)$ over RR, we generated data from the same model A in
(\ref{Simulation-Model}). We first fix $n=100$ and let $p$ vary in
$\{10, 20, \ldots, 100, 200, \ldots 1000\}.$ And then we fix
$p=100$ and let $n$ vary in $\{10, 20, \ldots, 100, 200, \ldots,
1000, 2000, \ldots, 10000\}.$ For each setting, we recorded the
CPU computing time for RR and $\mbox{RR}(d; \lambda)$ averaged
from three simulation runs. We have set the search range for
$\lambda$ as $\{0.1, 0.2, \ldots, 100\}.$ The results are plotted
in Figure \ref{fig02}(a) and \ref{fig02}(b). It can be seen that
$\mbox{RR}(d; \lambda)$ is much faster than RR, especially when
either $p$ or $n$ gets large.

\subsection{Comparing PCR with $\mbox{PCR}(d; c)$ and $\mbox{PCR}(\gamma; c)$}
\label{sec-simulation-PCR}

Next we compare the two WOCR variants that are close to PCR. Data of dimension $n=500$ and $p=50$ are generated from Model A in (\ref{Simulation-Model}), yet with two sets of coefficients $\mathbf{b}$ given as follows:
\begin{enumerate}[(i)]
\item $\mathbf{b} =\left[5, 5, 5, 5, 5, 0, 0,  \ldots,  0 \right]^T \in \mathbb{R}^{50};$

\item $\mathbf{b} =\left[0, 0, 0, 0, 5, 0, 0, \ldots,  0 \right]^T \in \mathbb{R}^{50}.$
\end{enumerate}
The first set (i) fits perfectly to ordinary PCR and hence $\mbox{PCR}(d; c)$ with number of useful components being 5, while the second set (ii) corresponds to the situation where the response is only associated with the fifth principal components, a scenario that fits best to $\mbox{PCR}(\gamma; c).$ Recall that the shape parameter $a$ in both $\mbox{PCR}(d; c)$ and $\mbox{PCR}(\gamma; c)$ is fixed at a relatively larger value. Concerning its choice, we consider four values $a \in \{5, 10, 50, 100\}.$ A total of 200 simulation run is made for each configuration. For each simulated data set, the ordinary PCR is fit with minimum cross-validated error, as implemented in R package \textbf{pls} while $\mbox{PCR}(d; c)$ and $\mbox{PCR}(\gamma; c)$ are fit with minimum BIC.  Figure \ref{fig03} plots the number of components selected by each method via boxplot and the MSE for predicting an independent test data set of $n'=500$ generated from the same model setting via mean plus/minus standard error bar plot.

\begin{figure}[h]
\centering
  \includegraphics[scale=0.60, angle=0]{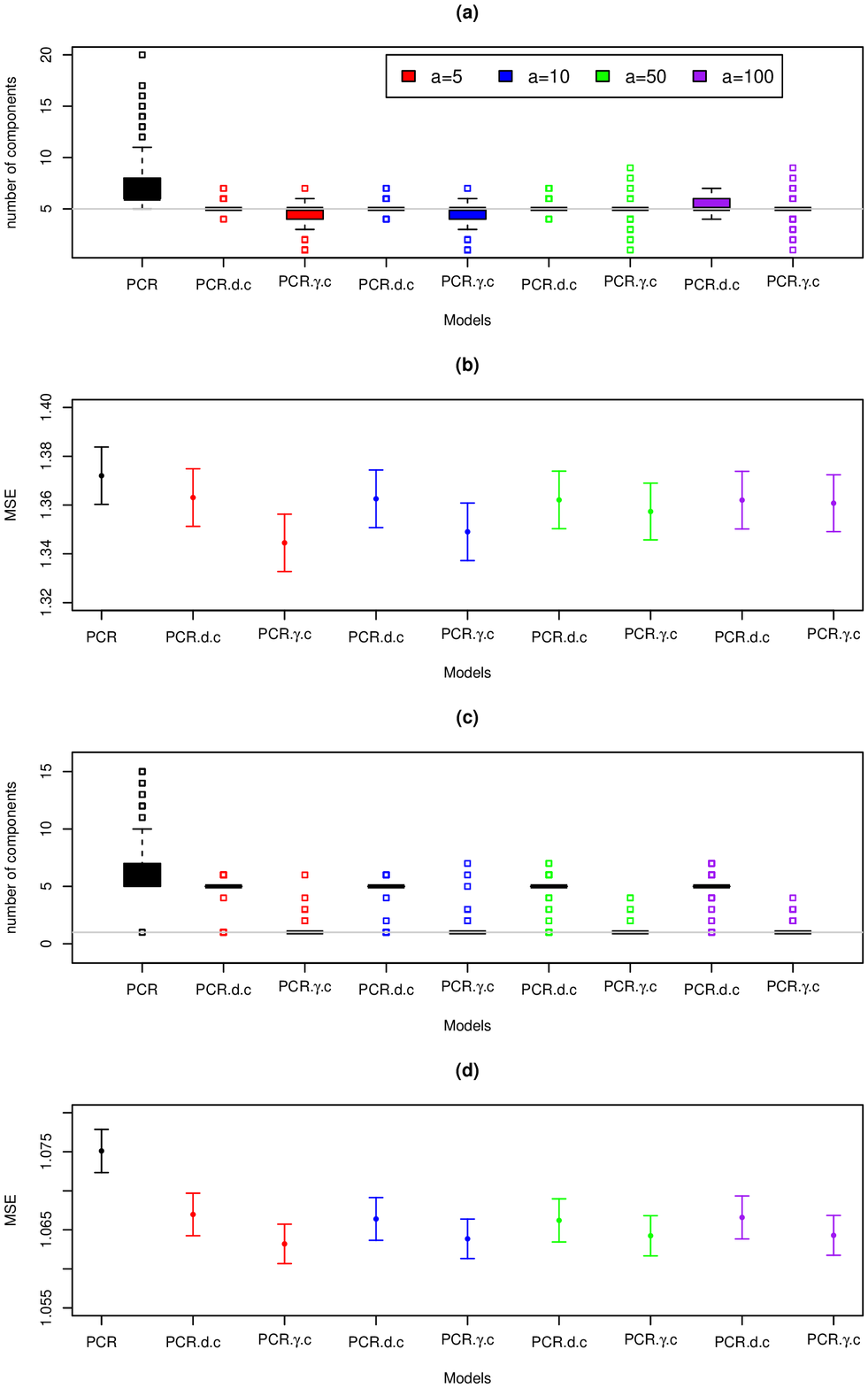}
  \caption{Comparing principal components regression (PCR) and its WOCR variant $\mbox{PCR}(d; c)$ and $\mbox{PCR}(\gamma; c)$ in terms of selecting the number of components and prediction MSE. Data are generated from Model A in (\ref{Simulation-Model}) with $\mathbf{b} =\left[5, 5, 5, 5, 5, 0, 0,  \ldots,  0 \right]^T \in \mathbb{R}^{50}$ in panels (a) and (b) and
$\mathbf{b} =\left[0, 0, 0, 0, 5, 0, 0, \ldots,  0 \right]^T \in
\mathbb{R}^{50}$ in panels (c) and (d), respectively.
\label{fig03}}
\end{figure}

It can be seen that PCR substantially overfits in both model
settings, resulting in high prediction errors as well. In the
first scenario (i), $\mbox{PCR}(d; c)$ and $\mbox{PCR}(\gamma; c)$
both do well with similar performance. In the second scenario
(ii), $\mbox{PCR}(d; c)$ fails in identifying the correct
principal components while $\mbox{PCR}(\gamma; c)$ remains
successful by switching the ordering from singular values $d_j$ to
regression coefficients $\gamma_j^2.$ For the different $a$
choices, the performance of $\mbox{PCR}(d; c)$ and
$\mbox{PCR}(\gamma; c)$ is quite stable with some minor
variations.

\subsection{Predictive Performance Comparisons}
\label{sec-simulation-predictive}

To assess the predictive performance of WOCR models, we generate
data of size $n$ from two nonlinear models used in
\citep{Friedman.1991}, which are given as follows:
\begin{eqnarray}
\mbox{Model B:~~~~~~~~~}  y & =& 0.1 \exp(4x_1) + 4 \expit\{20 (x_2 - 0.5)\} + 3 x_3 + 2 x_4 + x_5 + \varepsilon;  \label{modelB} \\
\mbox{Model C:~~~~~~~~~}  y & =& 10 \sin(\pi x_1 x_2) + 20 (x_3 - 0.5)^2 + 10 x_4 + x_5 + \varepsilon. \label{modelC}
\end{eqnarray}
The covariates of dimension $p$ are independently generated from the uniform[0,1] distribution and the random error term follows $\mathcal{N}(0, 1).$
In both models, only the first five predictors are involved in the mean response function. Two choices of $p \in \{5, 50\}$ are considered with $n=500.$ For each simulated data set, ridge regression, PCR, and six WOCR variants in Table \ref{tbl-WOCR-Models} are applied with default or recommended settings. In particular, we fix the scale parameter $a=50$ in $\mbox{PCR}(d; c)$ and $\mbox{PCR}(\gamma; c).$  To apply ridge regression, we have used $\lambda \in \{0.01, 0.02, \ldots, 200\}.$

Table \ref{tbl2} presents the prediction MSE (mean and SE) and the
median number of selected components by each method, out of 200
simulation runs. First of all, it can be seen that the ridge
regression appears to provide the worst results in terms of MSE. This is because of deficiencies involved in the current practice of ridge regression that computes ridge estimators for a discrete set of $\lambda$ within some specific range, which may not even include the true global GCV minimum. Comparatively, $\mbox{RR}(d; \lambda)$ provides a computationally efficient and reliable way of finding the `best' tuning parameter. We could have refit the ridge regression according to $\hat{\lambda}$ suggested by $\mbox{RR}(d; \lambda).$ Another interesting
observation is that $\mbox{RR}(\gamma; \lambda)$ tends to give
more favorable results than $\mbox{RR}(d; \lambda),$ because
sorting the components according to $|\gamma_j|$ borrows strength
from the association with the response.

Among PCR variants, neither $\mbox{PCR}(d; c)$ nor
$\mbox{PCR}(\gamma; c)$ performs well. On the basis of BIC, they
are aimed to find a parsimonious true model when the true model is
among the candidate models, which, however, is not the case here.
In terms of prediction accuracy, it can be seen that
$\mbox{RR}(\gamma; \lambda)$, $\mbox{PCR}(d; a, c)$, and
$\mbox{PCR}(\gamma; a, c)$ are highly competitive, all yielding
similar performance to PCR. Note that PCR determines the best
tuning parameter via 10-fold cross-validation, while
$\mbox{PCR}(d; a, c)$, and  $\mbox{PCR}(\gamma; a, c)$ are based
on a smooth optimization of GCV and hence are computationally
advantageous. In these simulation settings, PCR has selected all
components and hence simply amounts to the ordinary least square
fitting.

\begin{landscape}
\begin{table}[hp]
\renewcommand{\tabcolsep}{4.pt}
\renewcommand{\arraystretch}{1.2}
\caption{Comparison on predictive accuracy of ridge regression
(RR), principal components regression (PCR) with their six WOCR
variants. Data (with $n=500$) were generated from Models B and C.
Performance measures include the averaged MSE, the standard errors
of MSE, and the median number of selected components by each
method, out of 200 simulation runs for each configuration.}
 \vspace{.2in}
% \rule{5.5in}{.03cm}
\centering
\begin{tabular}{lllccccccccc} \hline \hline
    &       &       &   \multicolumn{9}{c}{Models}                   \\ \cline{4-12}
            &&& $\mbox{RR}(d; \lambda)$ & $\mbox{RR}(\gamma; \lambda)$  & RR    &\cellcolor[gray]{0.8}& $\mbox{PCR}(d; c)$  & $\mbox{PCR}(d; a, c)$ &   $\mbox{PCR}(\gamma; c)$ &   $\mbox{PCR}(\gamma; a, c)$  &   PCR \\ \hline
Model B & $p=5$ &   average--MSE    &   3.130   &   1.895   &   38.844  &\cellcolor[gray]{0.8}& 3.048   &   1.806 & 2.915   &   1.807   &   1.806   \\
    &&  SE--MSE &   0.0074  &   0.0094  &   1.1494  &\cellcolor[gray]{0.8}& 0.0972  &   0.0012  &   0.0577  &   0.0012  &   0.0012  \\
    &&  \# comps    &   5   &   5   &   5   &\cellcolor[gray]{0.8}& 4   &   5   &   2   &   5   &   5   \\ \cline{2-12}
    & $p=50$    &   average--MSE    &   2.485   &   2.057   &   6.051   &\cellcolor[gray]{0.8}& 2.499   &   2.059   &   2.773   &   2.075   &   2.062   \\
    && SE--MSE &    0.0071  &   0.0052  &   0.1654  &\cellcolor[gray]{0.8}& 0.0382  &   0.0053  &   0.0813  &   0.0058  &   0.0054  \\
    &&  \# comps    &   50  &   50  &   50  &\cellcolor[gray]{0.8}& 46  &   50  &   29  &   50  &   50  \\ \hline
Model C &   $p=5$   &   average--MSE    &   10.335  &   6.930   &   192.528 &\cellcolor[gray]{0.8}& 10.356  &   6.644   &   9.403   &   6.644 & 6.645   \\
    &&  SE--MSE & 0.0251    &   0.0345  &   8.8263  &\cellcolor[gray]{0.8}& 0.2830  &   0.0058  &   0.1559  &   0.0059  &   0.0059  \\
    &&  \# comps    &   5   &   5   &   5   &\cellcolor[gray]{0.8}& 4   &   5   &   2   &   5   &   5   \\ \cline{2-12}
    & $p=50$    &   average--MSE    &   9.058   &   7.223   &   54.781  &\cellcolor[gray]{0.8}& 9.257   &   7.210   &   9.617   &   7.226   &   7.229   \\
    &&  SE--MSE &   0.0230  &   0.0188  &   3.0581  &\cellcolor[gray]{0.8}& 0.1745  &   0.0181  &   0.2677  &   0.0185  &   0.0187  \\
    &&  \# comps    &   50  &   50  &   50  &\cellcolor[gray]{0.8}& 44  &   50  &   28.5    &   50  &   50  \\
\hline
\end{tabular}
% \rule{8.2in}{.03cm}
\label{tbl2}
\end{table}
\end{landscape}

\section{Real Data Examples}
\label{sec-real-data}

For further illustration, we apply WOCR to two well-known data sets,
which are \texttt{BostonHousing2} and \texttt{concrete}. The Boston housing data relates to prediction the median value of owner-occupied homes for 506 census tracts of Boston from the 1970 census. We used the corrected version \texttt{BostonHousing2} available from R package \textbf{mlbench} \citep{Leisch.2012}, with dimension $n=506$ observations and $p=17$ predictors.  The \texttt{concrete} data is available from the UCI Machine Learning
Repository (\url{https://archive.ics.uci.edu/ml/datasets/}). 
The goal of this data set is to predict the concrete compressive strength based on a few characteristics of the concrete. The data set has $n=1,030$ observations and $p=8$ continuous predictors.

Figure \ref{fig04} plots the singular values $d_j$ and the regression coefficients in absolute value $|\gamma_j|$ for both data sets. It can be seen that $d_j$ decreases gradually as expected. The bar plot of $|\gamma_j|$, however, shows different patterns. In the \texttt{BostonHousing2} data, the very first component is highly correlated with the response, while others shows alternate weak correlations. In the \texttt{concrete} data, the third component is most correlated with the response, followed by the 6th and 5th principal components. The first two components are only very weakly correlated. This data set shows a good example where the top components are not necessarily the most relevant components in terms of association with the response.

\begin{figure}[h]
\centering
  \includegraphics[scale=0.55, angle=270]{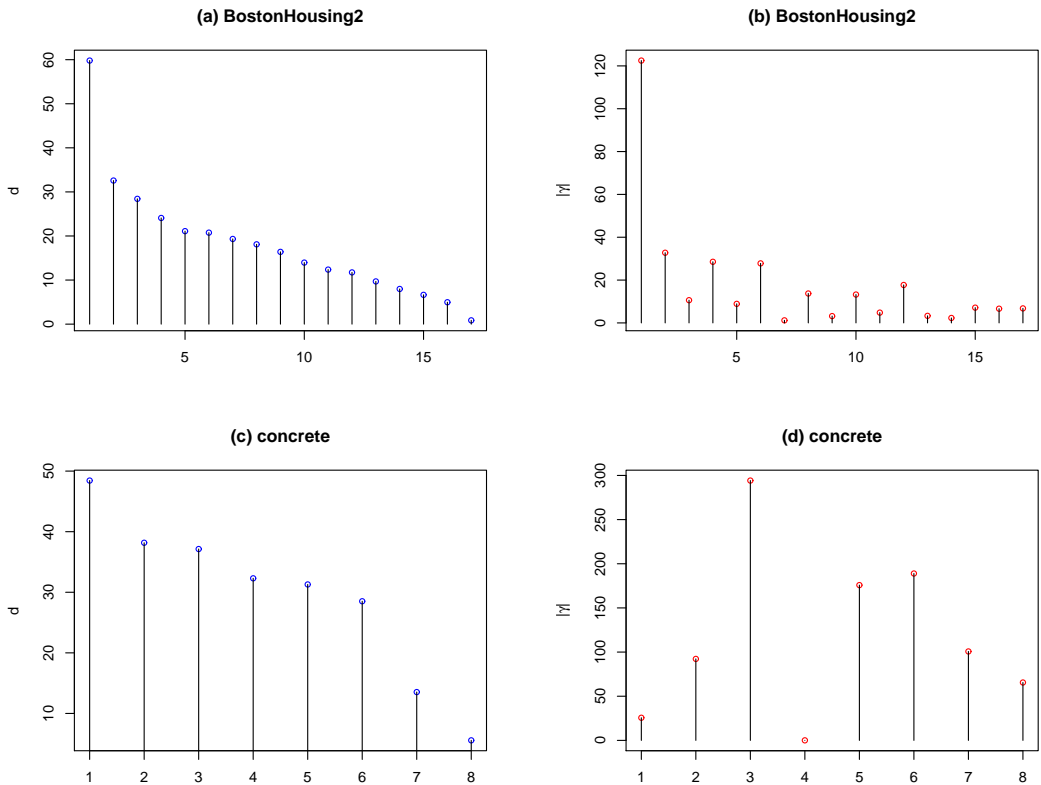}
  \caption{Bar plots of the singular values $d_j$ and the absoluate values of coefficients $|\gamma_j|$ for six real data sets. \label{fig04}}
\end{figure}

To compare different models, a unified approach is taken. We randomly
partition the data into the training set and the test set with a
ratio of approximately 2:1 in sample sizes. The training set is
used to construct models and then the constructed models are
applied to the test set for prediction. The default
settings in Table \ref{tbl-WOCR-Models} are used for each WOCR, while the default 10-fold CV method is used to select the best model for ridge regression and PCR. We repeat this entire procedure for 200 runs. The prediction MSE and the number of components for every method is recorded for each run. The results are summarized in Table \ref{tbl3-datasets}. 

% \begin{landscape}
\begin{table}[h]
\renewcommand{\tabcolsep}{7.5pt}
\renewcommand{\arraystretch}{1.2}
% \renewcommand{\baselinestretch}{1}
% \begin{small}
\caption{Comparison on predictive accuracy of ridge regression
(RR), principal components regression (PCR) with their WOCR
variants on two real data sets: \texttt{BostonHousing2} and \texttt{concrete}. The best performers are
highlighted in boldface.}
 \vspace{.2in}
% \rule{5.5in}{.03cm}
\centering
\begin{tabular}{lccccccc} \hline \hline
& \multicolumn{3}{c}{\texttt{BostonHousing2}} &\cellcolor[gray]{0.8}& \multicolumn{3}{c}{\texttt{concrete}} \\ \cline{2-4} \cline{6-8}
Method & average-MSE & SE-MSE & \# comps &\cellcolor[gray]{0.8}& average-MSE & SE-MSE & \# comps \\ \hline
Ridge	&	15.6630	&	0.1641	&	17	&\cellcolor[gray]{0.8}&	110.4587	&	0.4627	&	8	\\
$\mbox{RR}(d; \lambda)$	&	\textbf{15.6207}	&	0.1628	&	17	&\cellcolor[gray]{0.8}&	110.4581	&	0.4627	&	8	\\
$\mbox{RR}(\gamma; \lambda)$	&	15.6532	&	0.1652	&	17	&\cellcolor[gray]{0.8}&	\textbf{110.1557}	&	0.4630	&	8	\\ \hline
PCR	&	15.9962	&	0.1671	&	13.14	&\cellcolor[gray]{0.8}&	110.1968	&	0.4633	&	8	\\ 
$\mbox{PCR}(d; c)$	&	16.2175	&	0.1594	&	9.98	&\cellcolor[gray]{0.8}&	123.6064	&	0.5406	&	6	\\
$\mbox{PCR}(d; a, c)$	&	15.7529	&	0.1772	&	17	&\cellcolor[gray]{0.8}&	 110.1903	&	0.4633	&	8	\\
$\mbox{PCR}(\gamma; c)$	&	21.6077	&	0.1793	&	1	&\cellcolor[gray]{0.8}&	137.0467	&	2.0446	&	2.99	\\
$\mbox{PCR}(\gamma; a, c)$	&	\textbf{15.6596}	&	0.1747	&	17	&\cellcolor[gray]{0.8}&	\textbf{110.1852}	&	0.4631	&	8	\\
\hline
\end{tabular}
% \rule{8.2in}{.03cm}
\label{tbl3-datasets}
% \end{small}
\end{table}
% \end{landscape}

While most methods provide largely similar results, some details are noteworthy. For ridge regression, $\mbox{RR}(d; \lambda)$ outperforms the original ridge regression slightly but it is much faster in computation time. Comparatively, $\mbox{RR}(\gamma; \lambda)$ improves the prediction accuracy by basing the weights on $\gamma_j$'s for the \texttt{concrete} data, where the top components are not the most relevant to the response as shown in Figure \ref{fig04}.  
Among the PCR models, both $\mbox{PCR}(d; a, c)$ and  $\mbox{PCR}(\gamma; a, c)$ are among top performers in terms of prediction. 

Neither $\mbox{PCR}(d; c)$ nor $\mbox{PCR}(\gamma; c)$ perform as well as others in terms of prediction accuracy owing to their different emphasis. 
Concerning component selection, $\mbox{PCR}(\gamma; c)$ yields simpler models than $\mbox{PCR}(d; c)$ and PCR. This is determined by the nature of each method and data set. Referring to Figure  \ref{fig04}, $\mbox{PCR}(\gamma; c)$ clearly helps extract parsimonious models with simpler structures.

\section{Discussion}
\label{sec-discussion}

We have proposed a new way of constructing predictive models based on
orthogonal components extracted from the original data. The
approach makes good use of the natural monotonicity associated
with those orthogonal components. It allows efficient
determination of the tuning parameters. The approach results in
several interesting alternative models to RR and PCR. These new
variants make improvement on either predictive performance or
selection of the components. Overall speaking, $\mbox{RR}(\gamma;
\lambda)$, $\mbox{PCR}(d; a, c)$, and  $\mbox{PCR}(\gamma; a, c)$
are highly competitive in terms of predictive performance.
$\mbox{PCR}(\gamma; c)$ better aims for model parsimony  by making
selection on the basis of association with the response.

WOCR can be implemented with more flexibility. First of all, we
have advocated the use of logistic or expit function in regulating
the weights. The logistic function $\expit\{a(x-c)\}$ is
rotationally symmetrical about the point $(c, 0.5)$. To have more
flexible weights, we may consider a generalized version of the
expit function, $\mbox{gexpit}(x; a, b, c) = 1/ \left[ 1 + b \,
\exp \{-a(x-c)\} \right].$ The range of the gexpit function
remains (0, 1). Since its value at $x=c$ is now $1/(1+b),$ the
parameter $b >0$ changes the rotational symmetry unless $b=1.$
Secondly, selecting the number of principal components is a major
concern in PCR. We have used BIC in both $\mbox{PCR}(d; c)$ and
$\mbox{PCR}(\gamma; c)$ for this purpose. BIC is derived in the
fixed dimensional setting (i.e., fixed $p$ and $n \rightarrow
\infty$). It is worth noting that the dimension in the WOCR family
is $m$ instead of $p$. If $m$ is close to $n$, the modified or
generalized BIC (see, e.g., \citeauthor{Chen.2008},
\citeyear{Chen.2008}) can be used instead. In particular, the
complexity penalty $\left[\ln\{\ln(m)\} \ln(n)\right]$ suggested
by \citet{Wang.2009} to replace $\ln(n)$ in (\ref{BIC-c-PCR}) for
diverging dimensions fits well for WOCR models since the dimension
$m$ cannot exceed $n$. If there is prior information or belief
that the optimal $k$ is less than some pre-specified number, it is
helpful to further restrain the search range of $c$ on the basis of $\{d_j: j=1, \ldots, m\}.$

The WOCR model framework generates several future research
revenues. First of all, WOCR can be directly applicable to
regression with components after a varimax rotation
\citep{Kaiser.1958}. WOCR can also be extended to PLSR and CR
models. In those approaches, extraction of the orthogonal
components takes associations with the response into
consideration; thus both matrices $\mathbf{A}$ and $\mathbf{W}$
relate to $\mathbf{y}.$ To select the tuning parameter, $v$-fold
cross validation can be conveniently used on the basis of
Equations (\ref{beta-WOCR}) and (\ref{y-pred}). To implement
pre-tuning, finding the degrees of freedom involved in these
approach becomes more complicated but remains doable by following
\citet{Kramer.2011}. The weighting and pre-tuning strategy
introduced in WOCR may help make improvement in terms of
predictive accuracy, computational speed, and model parsimony for
these models. Secondly, the simulation results for Model B in
(\ref{modelB}) and Model C in (\ref{modelC}) with $p=50$ presented
in Section \ref{sec-simulation-predictive}  highlight the variable
selection issue in high-dimensional modeling. To this end,
\citet{Bair.2006} considered a univariate screening step;
\citet{Ishwaran.2014} showed the generalized ridge regression
\citep{Hoerl.1970} can help suppress the influence of unneeded
predictors in certain conditions. Both approaches may be
incorporated into WOCR to improve its predictive ability. Finally,
WOCR can be extended to generalized linear models, e.g., via a
local quadratic approximation of the log-likelihood function. The
kernel trick (see, e.g., \citealt{Rosipal.2001},
\citealt{Rosipal.2002}, and \citealt{Lee.2013}) can be integrated
into WOCR as well.

\iffalse
\vspace{.5in}
\appendix
\begin{center}
{\Large APPENDIX}
\end{center}
%-----------------------------------------------------------------------------------------------------------------------------------------%
\section{Proofs}
\fi



% \vspace{.3in}

% \newpage

\end{document}